\documentclass{article}

\usepackage{microtype}
\usepackage{graphicx}
\usepackage{subfigure}
\usepackage{booktabs} 
\usepackage{hyperref}
\usepackage{algpseudocode} 
\usepackage[accepted]{icml2024}

\usepackage{amsmath}
\usepackage{amssymb}
\usepackage{mathtools}
\usepackage{amsthm}
\usepackage{mathrsfs}
\usepackage{multicol}

\theoremstyle{plain}
\newtheorem{theorem}{Theorem}[section]

\newtheorem{lemma}[theorem]{Lemma}

\newtheorem{corollary}[theorem]{Corollary}
\theoremstyle{definition}
\newtheorem{definition}[theorem]{Definition}

\theoremstyle{remark}

\newcommand{\Set}[1]{\mathcal #1}
\newcommand{\R}{{\mathbb R}}%
\newcommand{\E}{{\mathbb E}}%

\usepackage{color,soul}

\usepackage[textsize=tiny]{todonotes}

\icmltitlerunning{To the Max: Reinventing Reward in Reinforcement Learning}

\begin{document}

\twocolumn[
\icmltitle{To the Max: Reinventing Reward in Reinforcement Learning}
\icmlsetsymbol{equal}{*}

\begin{icmlauthorlist}
\icmlauthor{Grigorii Veviurko}{delft}
\icmlauthor{Wendelin B\"ohmer}{delft}
\icmlauthor{Mathijs de Weerdt}{delft}
\end{icmlauthorlist}

\icmlaffiliation{delft}{Delft University of Technology}
\icmlcorrespondingauthor{Grigorii Veviurko}{g.veviurko@tudelft.nl}

\icmlkeywords{Reinforcement Learning, Machine Learning, ICML}

\vskip 0.3in
]



\printAffiliationsAndNotice{} 

\begin{abstract}
In reinforcement learning (RL), different reward functions can define the same optimal policy but result in drastically different learning performance. For some, the agent gets stuck with a suboptimal behavior, and for others, it solves the task efficiently. Choosing a good reward function is hence an extremely important yet challenging problem. 
In this paper, we explore an alternative approach for using rewards for learning. We introduce \textit{max-reward RL}, where an agent optimizes the maximum rather than the cumulative reward. 
Unlike earlier works, our approach works for deterministic and stochastic environments and can be easily combined with state-of-the-art RL algorithms. In the experiments, we study the performance of max-reward RL algorithms in two goal-reaching environments from Gymnasium-Robotics and demonstrate its benefits over standard RL. The code is available at \url{https://github.com/veviurko/To-the-Max}.
\end{abstract}

\section{Introduction}
Reinforcement Learning (RL) is a learning paradigm where an intelligent agent solves sequential decision-making problems through trial and error. The main objective that an RL agent learns to optimize is the \textit{cumulative return}, i.e., a discounted sum of the \textit{rewards}. This makes the reward a crucial element of the problem, as it defines the optimal decision-making policy that the agent will try to learn.  

It is well known \cite{ng1999policy} that there are infinitely many ways to define the reward function under which a desired policy is optimal. Practically, however, these rewards often result in drastically different learning processes.
For example, many major successes of RL required meticulous engineering of the reward: by hand \cite{berner2019dota} or by learning it from a human example \cite{vinyals2019grandmaster}. Hence, designing a reward function that enables learning and corresponds to a certain optimal policy is a challenging problem in modern reinforcement learning.  

In many RL problems, the true reward function is \textit{sparse}, i.e., only successful completion of the task is rewarded. In particular, the sparse reward is characteristic to \textit{goal-reaching} problems where the agent needs to enter the goal state \cite{plappert2018multi, florensa2018automatic, ghosh2020learning}. Sparse reward problems are notoriously hard to solve with standard RL.
A popular and simple solution is to introduce a dense surrogate reward that represents some sort of distance between the agent and the goal \cite{gymnasium2023github, gymnasium_robotics2023github}.
However, this approach is very sensitive and should be carefully tailored to each problem individually, in order to not change the induced optimal policy. {Specifically, this dense artificial dense reward should \textit{a)} increase when the agent gets closer to the goal, and \textit{b)} not distract the agent from the reaching the goal. Designing a function that satisfies both criteria can be tricky for a human expert, as it requires estimating the (discounted) cumulative returns in various states.}

In this work, we propose \textit{max-reward RL}, where the agent optimizes the maximum reward achieved in the episode rather than the cumulative return. {This paradigms makes the reward design process much more intuitive and straightforward, as it only requires that ``better'' states correspond to larger rewards. Hence, as long as the goal-reaching action has the highest reward, the optimal policy does not change. Besides simplifying the reward design, the maximum reward objective can also be easier to optimize for. In standard RL, learning a value of a non-terminal state involves bootstrapping, and hence has a moving target. In max-reward RL, bootstrapping does not happen when the immediate reward is not smaller than the largest reward explored so far. Therefore, max-reward RL bootstraps less and hence, potentially, learns better.}

One of the key properties of the cumulative return is that it satisfies the Bellman equation \cite{bellman1954some} and hence can be efficiently approximated and optimized by iteratively applying the Bellman operator. To make the max-reward RL approach viable, an analogous learning rule is required. However, \citet{cui2023reinforcement} prove that naively changing summation into a max operator in the standard Bellman update rule works \textit{only} in a deterministic setting and hence cannot be used in most RL problems and algorithms.

Inspired by results from stochastic optimal control theory \cite{Kroner2018-lo}, this paper introduces a theoretically justified framework for max-reward RL in the general stochastic setting. We introduce a Bellman-like equation, prove the stochastic and deterministic policy gradient theorems, and reformulate some of the state-of-the-art algorithms (PPO, TD3) for the max-reward case. Using the Maze environment \cite{gymnasium_robotics2023github} with different surrogate dense rewards, we experimentally demonstrate that max-reward algorithms outperform their cumulative counterparts. Finally, experiments with a challenging Fetch environment \cite{gymnasium_robotics2023github} show the promise of max-reward RL in more realistic goal-reaching problems.

\section{Related work}
The first attempt to formulate max-reward RL was made by \citet{quah2006}, where the authors derived a learning rule for the maximum reward state-action value function.
However, as it was shown later \cite{gottipati2020maximum}, that work made a technical error of interchanging expectation and maximum operators. \citet{gottipati2020maximum} corrected this error, but the value functions learned via their approach differ from the expected maximum reward if stochasticity is present.
Independently, \citet{wang2020planning} derived a similar method in the context of planning in deterministic Markov Decision Processes (MDPs).
Later, \citet{cui2023reinforcement} demonstrated that the presence of stochasticity poses a problem not only for the max-reward RL but also for other non-cumulative rewards. 

There exists a parallel branch of research that (re)discovered maximum reward value functions in the context of safe RL for reach-avoid problems \cite{fisac2014reachavoid}. In their work, \citet{fisac2019bridging} considered a deterministic open-loop dynamic system, where the agent's goal is to avoid constraint violations. The authors derived a contraction operator, similar to the one by \citet{gottipati2020maximum}, to learn the max-cost safe value function. 
\citet{hsu2021safelive} extended this approach to reach-avoid problems, where the goal is to reach the goal while not violating constraints.
Later, max-cost value functions were utilized within the safe RL context to learn the best-performing policy that does not violate the constraints \cite{yu2022reach}.
The main limitation of the three aforementioned works is the same as for \citet{gottipati2020maximum} -- their methods only apply to deterministic environments and policies.

Effective reward design is a long-standing challenge in reinforcement learning which dates back to at least as early as 1994 \cite{mataric1994reward}. In this paragraph, we briefly summarize the existing work related to the reward design problem. For further reading, we refer the reader to \citet{eschmann2021reward}. Some of the big successes of RL utilize a hand-designed reward function, e.g., in the game of DOTA \cite{berner2019dota} or robots playing soccer \cite{haarnoja2023learning}. However, manually designed rewards often lead to undesirable behavior \cite{deepmindblog}.
Alternatively, the reward can be designed in an automated fashion. For example, based on state novelty to encourage exploration \cite{tang2017exploration, pathak2017curiositydriven, burda2018exploration}, by learning it from the experiences \cite{trott2019keeping}, or by using human data \cite{ibarz2018reward}.

To conclude, reward design and reward shaping remain challenging topics. In this work, we propose a new way to think about the reward -- the max-reward RL framework. While self-sufficient in some cases, this approach can also be combined with various existing methods for reward design.

\section{Background}
\begin{figure*}[t!]
\vskip 0.2in
\begin{center}
\includegraphics[width=1.3\columnwidth]{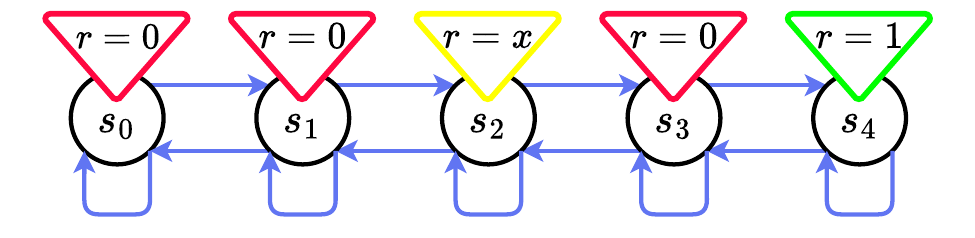}
\includegraphics[width=1.6\columnwidth]{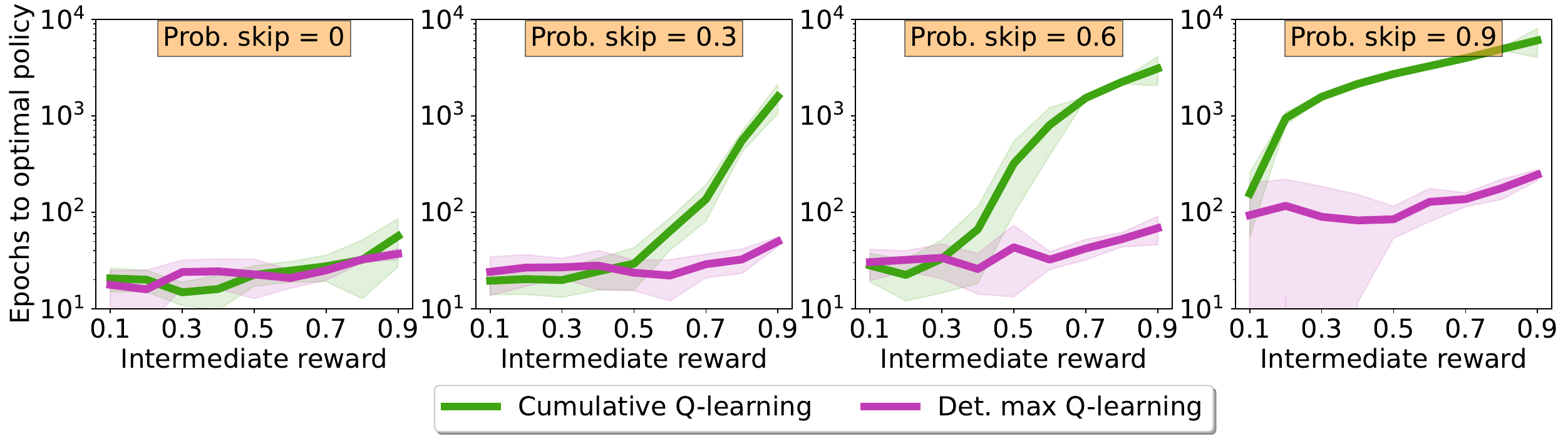}
\caption{Five-state chain MDP {with three actions (\textit{left, stay, right}) available in each state} and the training results for cumulative (in green) and max-reward (in violet) value iteration. The $y-$ axis is the number of training epochs to recover the optimal policy; the $x-$axis shows the values of the intermediate reward $x.$ Four panels correspond to different probabilities of skipping transitions into $s_4$ during training.}
\label{fig:chain-env}
\end{center}
\vskip -0.1in
\end{figure*}
We consider a standard reinforcement learning setup for continuous environments. An agent interacts with an MDP defined by a tuple $(\mathcal{S}, \mathcal{A}, R, P, p_0, \gamma),$ where $\mathcal{S}$ is the continuous \textit{state space}, $\mathcal{A}$ is the continuous \textit{action space}, and ${R:\mathcal{S}\times\mathcal{A}\times\mathcal{S}\to[0, \bar{R}]}$ is a non-negative and bounded \textit{reward function}.\footnote{Non-negativity of reward can be achieved in any MDP with bounded reward function.}
For each state-action pair, $(s,a)\in\mathcal{S}\times\mathcal{A}$, the transition function $P(\cdot|s, a)\in\mathcal{P}(\mathcal{S})$ is a probability density function (PDF) of the next state $s'$ and $p_0(\cdot)\in\Set{P}(\Set{S})$ is the PDF of the initial state $s_0$. Scalar $0\leq\gamma<1$ is the \textit{discount factor}. We use $\pi:\mathcal{S}\to\mathcal{P}(\mathcal{A})$ to denote a stochastic policy and $\mu:\mathcal{S}\to\mathcal{A}$ to denote a deterministic policy.
The time is discrete and starts at zero, i.e., $t\in\mathbb{N}\cup\{0\}$. For each timestep $t$, the state is denoted by $s_t$, the action by $a_t$, and the reward by $r_{t+1}:=R(s_t,a_t,s_{t+1}).$
Everywhere in the text, the expectation over policy, $\E_\pi$, denotes the expectation over the joint distribution of $s_t,a_t,r_{t+1}$ for $t\in\mathbb{N}\cup\{0\}$ induced by $\pi, P$, and $p_0.$ Sometimes, we use such notation as $\E_{x\sim\pi}$ (or just $\E_x$) to emphasize that the expectation is taken only over $x$.

In standard RL, the main quantity being optimized is the \textit{cumulative return}, defined as $G_t = \sum_{i=0}^\infty \gamma^ir_{t+1+i}$.
To maximize $\mathbb{E}_\pi\big[G_t\big]$, most RL algorithms learn state and/or state-action value functions defined as follows:
\begin{equation*}
    v^\pi(s)=\mathbb{E}_{\pi}\big[G_t \big|\substack{s_t=s}\big] \,,
    \quad 
    v^\ast(s)=\max_{\pi}v^\pi(s).
\end{equation*}
\begin{equation*}
    q^\pi(s, a)=\mathbb{E}_{\pi}\big[G_t \big|\substack{s_t=s\\a_t=a}\big] \,,
    \quad 
    q^\ast(s, a)=\max_{\pi}q^\pi(s, a).
\end{equation*}

Crucially, these functions are solutions to the corresponding \textit{Bellman equations}:
\begin{equation*}
v^\pi(s) = {\E}_{\substack{a_t\\s_{t+1}}}\big[r_{t+1} + \gamma v^\pi(s_{t+1})\big|\substack{s_t=s} \big]
\end{equation*}
\begin{equation*}
q^\pi(s, a) = {\E}_{\substack{s_{t+1}\\a_{t+1}}}
\big[r_{t+1} + \gamma q^\pi(s_{t+1}, a_{t+1})\big|\substack{s_t=s\\a_t=a} \big]
\end{equation*}
\begin{equation*}
q^\ast(s, a) = 
{\E}_{{s_{t+1}}}
\big[r_{t+1} + \gamma \max_{a'}q^\ast(s_{t+1},a') \big|\substack{s_t=s\\a_t=a} \big]
\end{equation*}
The defining feature of these equations is that they can be solved by repeatedly applying \textit{Bellman operators.} These operators are contractions and hence each of them has a unique fixed point that corresponds to one of the value functions above.
For example, the optimal state-action value function $q^\ast(s, a)$ is the fixed point of the \textit{Bellman optimallity operator} $\mathcal{T}^\ast$:
\begin{equation}
    \label{eq:bellman}
    \big(\mathcal{T}^\ast q\big)(s,a)= \mathbb{E}_{s_{t+1}}\big[r_{t+1} + \gamma \max_{a'}q(s_{t+1}, a')\big| \substack{\\s_t=s\\a_t=a} \big]
\end{equation}
The Bellman equation is foundational for all state-of-the-art RL algorithms as it allows training neural networks to approximate value functions. Therefore, for the max-reward framework to be useable, it is necessary to derive an analog of the Bellman equation. Below, we describe such an attempt made by \citet{gottipati2020maximum} and demonstrate that it is limited to purely deterministic problems.

\subsection{Deterministic max-reward RL}
\label{sec:det-max-r}
Instead of cumulative return, max-reward RL aims at optimizing the \textit{max-reward return}:
\begin{equation}
\label{eq:G_hat}
    \hat{G}_{t}= \max \big\{r_{t+1}, \gamma r_{t+2}, \gamma^2r_{t+3}\ldots\big\}
\end{equation}
{Similarly to cumulative returns, $\hat{G}_t$ uses the discount factor $\gamma$ which is necessary for learning with Bellman-like updates, as we show later}. To approximate $\mathbb{E}_\pi \big[\hat{G}_t\big]$, \citet{gottipati2020maximum} introduced the following definition of the state-action value functions:
\begin{equation*}
    \hat{q}^\pi_{det}(s,a)=\mathbb{E}_{\substack{s_{t+1}\\a_{t+1}}}\Big[
        r_{t+1} \lor \gamma q(s_{t+1}, a_{t+1})\Big|\substack{s_t=s\\a_t=a} \Big]
\end{equation*}
\begin{equation*}
    \hat{q}^\ast_{det}(s,a)=\mathbb{E}_{s_{t+1}}\Big[
        r_{t+1} \lor \gamma \max_{a'}q(s_{t+1}, a')\Big|\substack{s_t=s\\a_t=a} \Big]
\end{equation*}
where $\lor$ denotes the binary $\max$ operator, i.e., ${a\lor b:=\max\{a, b\}.}$ By construction, $\hat{q}^\ast_{det}$ and  $\hat{q}^\pi_{det}$ satisfy Bellman-like recursive equations. In their work, \citet{gottipati2020maximum} proved that the following operator is a contraction:
\begin{equation}
    \big(\hat{\mathcal{T}}^\ast_{det}q\big)(s, a) = 
    \mathop{\E}_{{s_{t+1}}}\big[r_{t+1} \lor
    \gamma \max_{a'}q(s_{t+1},a')\big| \substack{s_t=s\\a_t=a}\big]
    \label{eq:bellman-det}
\end{equation}
Therefore, $\hat{q}^\ast_{det}$ is the unique fixed point of $\hat{\Set{T}}_{det}^\ast$ and can be learned, e.g., with Q-learning.

\paragraph{Chain environment example.} Before going into the limitations of the approach above, we conduct a simple experiment to motivate the use of max-reward reinforcement learning. We show that max-reward RL is a better approach in a goal-reaching problem where the agent needs to learn to reach the goal state. Specifically, it dominates the standard cumulative RL when transitions into the goal state occur infrequently in the training data, which is often the case in larger-scale goal-reaching problems. 

Consider the five-state chain environment in Figure \ref{fig:chain-env}. Transitions leading into $s_4$ have reward of $1,$ transitions into $s_2$ have a reward parametrized with $x\in(0,1)$, and other rewards are zero. Hence, the optimal policy, concerning both max-reward and cumulative returns, is to go to $s_4$ and stay there.
We run tabular $Q-$value iteration algorithm using standard (Eq.~\eqref{eq:bellman}) and max-reward (Eq.~\eqref{eq:bellman-det}) Bellman operators for different values of the intermediate reward $x.$ In each training epoch, we iterate over all possible transitions. For each transition, we compute the target value using one of the Bellman operators and update the Q-table. Crucially, we \emph{randomly skip some of the transitions into $s_4$ with a certain probability}. In the experiment, we consider four values for the skip probability -- $p_{skip}\in\{0, 0.3, 0.6, 0.9\}.$ During training, when a transition into $s_4$ is sampled, the Q-table is updated with probability $1-p_{skip}$ and otherwise left unchanged. Transitions into other states are never skipped. In this way, we can control how often the agent is exposed to the transitions into the optimal state and thereby simulate problems where goal-reaching transitions are rarely encountered.

The results in Figure \ref{fig:chain-env} indicate that for larger values of the skip probability, the max-reward approach converges to the optimal policy significantly faster than the cumulative approach.
We believe that this phenomenon can be explained by differences in bootstrapping. In standard RL, the target for the $q-$value is a sum of the immediate reward and the $q-$value at the next timestep. Therefore, this target changes in each epoch until convergence. In the max-reward case, on the other hand, the target in the max-reward state is just the reward and does not change with time. This example suggests that the max-reward approach is a better choice in environments where the task of the agent is to reach the goal state.

\begin{figure}
    \centering
    \includegraphics[scale=0.45]{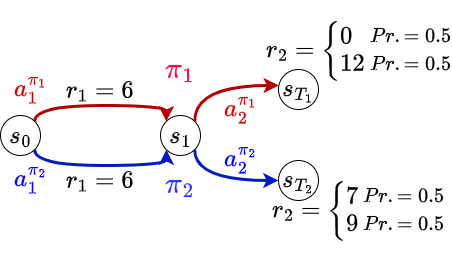}
    \vskip -0.2in
    \caption{A three-state MDP with deterministic transitions and stochastic rewards. Two different policies, $\pi_1$ and $\pi_2$, share the same first action $a_1$, but then have different $a_2$, thereby resulting in different reward distributions.}
    \label{fig:det-vs-stoch}
\end{figure}

\textbf{Issues when stochasticity is present}. Unfortunately, the max-reward approach described above has a serious theoretical drawback. Expanding the definition of $\hat{q}^\ast_{det}$ for more timesteps, we obtain a nested sequence of non-interchangeable $\lor$ and $\mathbb{E}$:
\begin{equation*}
 \hat{q}^\ast_{det}(s,a) = \E_{\pi^\ast}\Big[r_{t+1}\lor\gamma\E_{\pi^\ast}\big[r_{t+2}\lor\ldots\big]\Big|\substack{s_t=s\\a_t=a}\Big]
\end{equation*}
Using Jensen's inequality \cite{jensen1906fonctions}, we conclude the following:
\begin{equation}
    \label{eq:jensen}
    \hat{q}_{det}^\ast(s,a)\leq 
    \E_{\pi^\ast}\big[\hat{G}_t|\substack{s_t=s\\a_t=a}\big]
\end{equation}
When both the policy and the transition model are deterministic, Eq.~\eqref{eq:jensen} becomes an equality. However, if stochasticity is present, the value of $\hat{q}_{det}^\ast(s,a)$ is merely a lower bound of the expected return. Hence, it can induce suboptimal policies.

In Figure \ref{fig:det-vs-stoch}, we show an example where the policy maximizing $\hat{q}^\ast_{det}$ is suboptimal. The figure demonstrates a three-state MDP and two policies, $\pi_1$ (red arrows) and $\pi_2$ (blue arrows). Let $\gamma=1$ for simplicity. For the state $s_0$, the expected max-reward return is higher for the policy $\pi_1:$
\begin{equation*}
    \begin{aligned}
        \E_{\pi_1}[\hat{G}_0] & = \E_{\pi_1}[r_1\lor r_2] = 9 > \\ &  
    \E_{\pi_2}[\hat{G}_0] = \E_{\pi_2}[r_1 \lor r_2] = 8
    \end{aligned}
\end{equation*}
So $\pi_1$ is better in terms of the expected max-reward return, but the value functions have the following values:
\begin{equation*}
\hat{q}_{det}^{\pi_1}(s_0)=\E_{\pi_1}\big[r_1\lor \E_{\pi_1} [r_2]\big] = \E_{\pi_1}[r_1\lor 6] = 6
\end{equation*}
\begin{equation*}
\hat{q}_{det}^{\pi_2}(s_0)=\E_{\pi_2}\big[r_1\lor \E_{\pi_2} [r_2]\big] = \E_{\pi_2}[r_1\lor 8] = 8
\end{equation*}
Based on the values of $\hat{q}_{det}^\pi$, we would conclude that $\pi_2$ is better, which we already showed to be incorrect. This example demonstrates that even in a simple stochastic environment, the operator $\hat{T}_{det}^\ast$ can lead to incorrect policies. Therefore, it is an open question whether there exists a Bellman-like operator that would enable learning max-reward returns in the stochastic setting. 

\section{Max-reward RL}
\label{sec:max-reward RL}
In this section, we introduce a novel approach to max-reward RL that is theoretically sound, works for both stochastic and deterministic cases, and can be combined with state-of-the-art RL algorithms.
First, we expand the definition of the max-reward return given in Eq.~\eqref{eq:G_hat}:
\begin{equation}
\label{eq:bellman_G_max}
\E_{\pi}[\hat{G}_t] = \E_{\pi}\big[r_{t+1} \lor \gamma \hat{G}_{t+1}\big]    
\end{equation}
Since $\E$ and $\lor$ do not commute, it is impossible to extract the term $\E_\pi\big[G_{t+1}\big]$ on the right-hand side of Eq.~\eqref{eq:bellman_G_max}. Because of that, we cannot obtain an equation involving only $\E_{\pi}\big[G_t\big],\;\E_{\pi}\big[G_{t+1}\big]$, and $r_{t+1}.$
Instead, we will utilize an approach from stochastic optimal control theory \cite{Kroner2018-lo} and define the max-reward value function using an auxiliary variable that allows propagating information between timesteps:
\begin{definition} \label{def:value-functions}
Let $y\in\mathbb{R}$ be an auxiliary real variable. The \emph{max-reward value functions} are defined as follows: 
    \begin{equation*} 
    \hat{v}^\pi(s, y) = \E_{\pi}[y\lor\hat{G}_t\big|\substack{s_t=s}]
\end{equation*}
\begin{equation*}
    \hat{q}^\pi(s, a, y) =  \E_{\pi}[y\lor\hat{G}_t \big|\substack{s_t=s\\a_t=a}]
\end{equation*}
\end{definition}
Since reward is lower-bounded, $r_{t+1}\geq 0$, we can always recover the expected max-reward return  $\E_{\pi}\big[\hat{G}_t\big]$ by substituting $y=0$ into the value functions:
\begin{equation}   
    \label{eq:lb}
    \begin{aligned}
    \hat{v}^\pi(s, 0) = 
    \E_{\pi}[\hat{G}_t\big|\substack{s_t=s}] 
    \\
    \hat{q}^\pi(s, a, 0) = \E_{\pi}[\hat{G}_t\big|\substack{s_t=s\\a_t=a}]
    \end{aligned}
\end{equation}
Hence, if we find an efficient method of learning the max-reward value functions, we will be able to optimize $\E_{\pi}[\hat{G}_t]$.

The auxiliary variable $y$ is crucial when dealing with the max-reward returns. When we look at the value of the state $s'$ from the perspective of state $s,$ we must consider the immediate reward $r=r(s, a, s').$ Specifically, we should treat low reward trajectories from $s'$ as if they still yield the reward of $r.$ Expanding upon this observation, we conclude that maximization of the maximum reward requires propagating information about the past rewards. This is achieved via the auxiliary variable $y.$

By combining the definition of the max-reward value functions with Eq.~\eqref{eq:bellman_G_max}, we obtain the following recursive equations:
\begin{lemma} \label{lemma:bellman-eq}
    Let $y\in\R$ and let $y':=\frac{R(s,a,s_{t+1})\lor y}{\gamma}$. Then, the max-reward value functions are subject to the following Bellman-like equations:
    \begin{equation*}
    \hat{v}^\pi(s, y) = \gamma\E_{\substack{a_t\\s_{t+1}}} \big[y'\lor\hat{v}^\pi(s_{t+1}, y')\big|\substack{s_t=s}\big]
    \end{equation*}
    \begin{equation*}
    \hat{q}^\pi(s, a, y) = \gamma\E_{\substack{s_{t+1}\\a_{t+1}}} \big[y'\lor \hat{q}^\pi(s_{t+1}, a_{t+1}, y')\big| \substack{s_t=s\\a_t=a}\big]
    \end{equation*}
\end{lemma}
Proof of this lemma, as well as all other proofs, can be found in Appendix \ref{proofs}. The extra term ${y' \lor}$ might seem redundant, but it is important since it enforces the boundary conditions. Without it, the functions $v\equiv 0$ and $q\equiv 0$ would be solutions to these equations.
Using Lemma \ref{lemma:bellman-eq}, we can define Bellman-like operators for the max-reward value functions: 
\begin{definition}
    Let $v: \mathcal{S}\times \mathbb{R}\to \mathbb{R},\: q:\mathcal{S}\times\mathcal{A}\times \mathbb{R}\to \mathbb{R} $ be real-valued functions and let $y':=\frac{R(s,a,s_{t+1})\lor y}{\gamma}$. Then, the \emph{max-reward Bellman operator} $\hat{\mathcal{T}}^\pi$ is defined as follows:
    \begin{equation*}
    \hat{\mathcal{T}}^\pi v (s, y) := \gamma\E_{\substack{a_t\\s_{t+1}}} \big[y'\lor v(s_{t+1}, y')\big|\substack{s_t=s}\big]
    \end{equation*}
    \begin{equation*}
    \hat{\mathcal{T}}^\pi q(s,a,y) := \gamma\E_{\substack{s_{t+1}\\a_{t+1}}} \big[y'\lor q(s_{t+1}, a_{t+1}, y')\big| \substack{s_t=s\\a_t=a}\big]
    \end{equation*}
\end{definition}
In the following theorem, we prove that this operator is a contraction and that the max-reward state and state-action value functions are its fixed points.
\begin{theorem}
\label{theorem:contraction}$\hat{\mathcal{T}}^\pi$ is a $\gamma-$contraction with respect to the $L_\infty$ norm, and $\hat{v}^\pi$(or $\hat{q}^\pi$) is its fixed point.
\end{theorem}
Theorem \ref{theorem:contraction} implies that the max-reward value functions can be learned in the same way as the standard value functions -- by sampling from the environment and applying Bellman operators.
In the next section, we define the objective function of the max-reward RL problem and discuss how the presence of the auxiliary variable $y$ impacts the notion of optimal policy.

\subsection{Max-reward objective}
\label{sec:max-reward-obj}
Similarly to standard RL, the main objective in the max-reward RL problems is to maximize the expected (max-reward) return from the initial state, defined as follows:
\begin{equation}
        \hat{J}(\pi) = \E_{s_0\sim p_0}\big[\hat{v}^\pi(s_0, 0)\big]
\end{equation}
Then, the optimal policy is naturally defined as :
\begin{equation}
\label{eq:opt-pi}
\pi^\ast = \arg\max_{\pi}\hat{J}(\pi).
\end{equation}
To better understand the properties of the max-reward optimal policy, consider again the MDP in Figure \ref{fig:det-vs-stoch}. Let $\gamma=1.$ Then, the values of the objective function for $\pi_1$ and $\pi_2$ can be computed as follows:
\begin{equation*}
    \hat{J}(\pi_1) = \E_{\pi_1}[6\lor r_2] = 9
\end{equation*}
\begin{equation*}
    \hat{J}(\pi_2) = \E_{\pi_2}[6\lor r_2] = 8
\end{equation*}
Hence, $\pi_1$ is optimal. However, if we consider the max-reward return from $t=1,$ we have 
\begin{equation*}
    \E_{\pi_1}[G_1]=\frac{12+0}{2}=6\quad 
    \E_{\pi_2}[G_1]=\frac{9+7}{2}=8
\end{equation*}
and hence $\pi_2$ obtains higher expected max-reward return starting at $s=s_1.$ Seemingly, there is a contradiction: $\pi_1$ is optimal but $\pi_2$ is better from the state $s_1.$ However, the explanation is simple:
the maximum reward is the highest reward encountered anywhere along the trajectory. An optimal decision thus not only depends on the current state, as with the cumulative reward, but also on the maximum reward that has been acquired thus far.
In the example, if we start from $s_1$, then we haven’t encountered any reward yet. Hence, following $\pi_1$, we will have $r_2=0$ as the maximum reward half of the time. If we start from $s_0$, we receive a reward of $r_1=6$ when going to $s_1$. Then, the maximum reward will not be lower than 6, even if we get $r_2=0$. Thus, we conclude:
\begin{center}
    \emph{In max-reward RL, the optimal policy $\pi^\ast$ maximizing $\hat{J}(\cdot)$ should depend not only on the current state, but also on the rewards obtained so far.} 
\end{center}

To formalize this observation, we introduce additional notation. We define the \textit{extended state space} as $\hat{\mathcal{S}}:=\Set{S}\times\R$ and we denote extended states by ${\hat{s}=(s,y),\; s\in\Set{S},\;y\in\R.}$ Then, for an extended state $(s,y)\in\hat{\Set{S}}$ and for an action $a\in\Set{A},$ the \textit{extended transition model}  $\hat{P}(\cdot, \cdot|s,y,a)$ is a PDF over ${(s', y')\in\hat{\Set{S}}},$ defined as 
\begin{equation*}
    \hat{P}(s', y'|s,y,a)=P(s'|s,a)\,\delta\Big(y'-\frac{R(s,a,s')\lor y}{\gamma}\Big)
\end{equation*}
where $\delta(\cdot)$ is the Dirac delta function. The initial distribution of $(s_0,y_0)$ is given by $\hat{p}_0(s_0,y_0)=p(s_0)\delta(y_0)$ thereby ensuring $y_0\equiv 0$. Combining everything, we introduce the following definition:
\begin{definition}
    \label{def:ext-mdp}
    Let $M=(\mathcal{S}, \mathcal{A}, R, P, p_0, \gamma)$ be an MDP. Then, the \textit{extended max-reward MDP} is an MDP $\hat{M}$ given by the tuple $(\hat{\mathcal{S}}, \mathcal{A}, R, \hat{P}, \hat{p_0}, \gamma).$
\end{definition}
Essentially, the extended MDP defined above tracks the (inversely) discounted maximum reward obtained so far. For example, if the maximum reward so far is $r_1$, then the extended state at timestep $t$ is $(s_t, \frac{r_1}{\gamma^{t}})$. Hence, to improve $\hat{J}$, we need $r_{t+1} > \frac{r_1}{\gamma^{t}}$.

Using the notion of extended MDP, we can redefine policy for the max-reward RL:
\begin{definition}
    Let $M$ be an MDP and let $\hat{M}$ be its induced extended max-reward MDP. Then, any policy $\hat{\pi}$ in $\hat{M}$ is an \textit{extended max-reward policy.} 
\end{definition}
After we have defined optimality in the max-reward sense, we can introduce the max-reward Bellman optimality operator:
\begin{definition}
\label{def:opt-maxr-bellman}
    Let $q:\mathcal{S}\times\mathcal{A}\times \mathbb{R}\to \mathbb{R} $ be a real-valued function and let $y':=\frac{R(s,a,s_{t+1})\lor y}{\gamma}$. Then, the \emph{max-reward Bellman optimality operator} $\hat{\mathcal{T}}^\ast$ is defined as follows:
    \begin{equation*}
    \hat{\mathcal{T}}^\ast q(s,a,y) := \gamma\E_{{s_{t+1}}} \big[y'\lor \max_{a'}q(s_{t+1}, a', y')\big| \substack{s_t=s\\a_t=a}\big]
    \end{equation*}
\end{definition}
Similarly to $\hat{T}^\pi,$ this operator is also a contraction:
\begin{theorem}
\label{theorem:opt-contraction}$\hat{\mathcal{T}}^\ast$ is a $\gamma-$contraction with respect to the $L_\infty$ norm, and $\hat{q}^\ast$ is its fixed point.
\end{theorem}

We have most of the pieces of the max-reward RL framework. We established that it operates on the extended max-reward MDP $\hat{M}$, where the extended states preserve information about the past rewards. Then, both the max-reward optimal and on-policy value functions can be learned by sampling transitions from $\hat{M}$. Therefore, all DQN-based methods \cite{mnih2013playing} can be used under the max-reward RL paradigm directly. However, most state-of-the-art RL algorithms utilize policies parametrized by neural networks. This is possible due to the policy gradient theorems \cite{sutton1999policy, silver2014deterministic}, as they allow estimating the objective function gradient with respect to the policy parameters via sampling. In the next section, we formulate and prove max-reward policy gradient theorems for both deterministic and stochastic extended max-reward policies. 

\subsection{Policy gradient theorems}
First, we define $\hat{p}^{\hat{\pi}}_t(s_0,y_0, s, y)$ -- the probability measure of arriving in the extended state $(s,y)$ after $t$ timesteps, starting from $(s_0, y_0)$ and executing the extended policy $\hat \pi$. Let 
$$\hat{P}^{\hat{\pi}}(s',y'|s,y)=\int_{a}\hat{\pi}(a|s,y)\hat{P}(s',y'|s,y,a)da$$ be the ``on-policy'' transition model. Then, $\hat{p}^{\hat{\pi}}_t(s_0,y_0, s, y)$ is defined as follows:
\begin{equation*}
\hat{p}_0(s_0,y_0,s,y)=\delta(s-s_0)\delta(y-y_0)
\end{equation*}
\begin{equation*}
\begin{aligned}
\hat{p}^{\hat{\pi}}_t(s_0,y_0,s,y)= \mathop{\int}_{\Tilde{s}, \Tilde{y}}&\hat{p}^{\hat{\pi}}_{t-1}(s_0,y_0,\Tilde{s},\Tilde{y})\,\hat{P}^{\hat{\pi}}(s,y|\Tilde{s},\Tilde{y})\,d\Tilde{s}\,d\Tilde{y}
\end{aligned}
\end{equation*}
The discounted stationary state distribution of an extended max-reward MDP is then given by
\begin{equation*}
    \hat{d}^{\hat{\pi}}(s, y) =\!\!\! \mathop{\int}_{s_0,y_0}\!\!\!\hat{p}_0(s_0,y_0)\sum_{t=0}^{\infty}\gamma^t\,\hat{p}^{\hat{\pi}}_t(s_0,y_0, s, y)\,ds_0\,dy_0.
\end{equation*}
As such, $\hat{d}^{\hat{\pi}}$ is not a distribution. However, it can be normalized into one by dividing it by $C=\int_{s,y}\hat{d}^{\hat{\pi}}(s, y)\,ds\,dy$.

Finally, we can formulate and prove the max-reward policy gradient theorems. Consider a neural network with weights $\theta$ that represents a stochastic policy. Then, we have the following result:

\begin{theorem}
    \label{th:stochastic-pg}
    Let $\hat{\pi}_\theta:\mathcal{S}\times\R\to\mathcal{P}(\mathcal{A})$ be a stochastic extended max-reward policy parameterized with $\theta$. Then, the following holds for $\nabla_\theta \hat{J}(\theta)$:
    \begin{equation*}
        \nabla_\theta\; \hat{J}(\theta) \propto \E_{\substack{(s,y)\sim\hat{d}^{\hat{\pi}}
        \\a\sim\hat{\pi}_\theta}}
\big[\hat{q}^{\hat{\pi}_\theta}(s,a,y)\nabla_\theta \ln\hat{\pi}_\theta(a|s,y)\big]
    \end{equation*}
\end{theorem}
The deterministic max-reward policy gradient follows from the stochastic version:
\begin{corollary}
    \label{th:deterministic-pg}
    Let $\hat{\mu}_\theta:\mathcal{S}\times\R\to\mathcal{A}$ be a deterministic extended max-reward policy parameterized with $\theta$. Then $\nabla_\theta \hat{J}(\theta)$ can be computed as follows:
    \begin{equation*}
        \nabla_\theta J(\theta) \propto
    \E_{\hat{d}^{\hat{\mu}}}
\big[\nabla_\theta \hat{\mu}_\theta(s,y)\nabla_a \hat{q}^{\hat{\mu}_\theta}(s,a,y)|_{a=\mu_\theta(s,y)}\big]
    \end{equation*}
\end{corollary}

\begin{figure}[b!]
    \centering
    \includegraphics[scale=0.65]{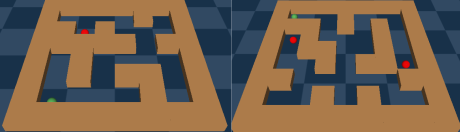}
    \caption{\textit{Left:} Single-goal maze, where the goal (red ball) is always in the same location.  \textit{Right:} Two-goals maze with two spawn locations of the goal (red balls).}
    \label{fig:mazes}
\end{figure}
\begin{figure*}[t!]
\begin{center}
\includegraphics[width=1.9\columnwidth]{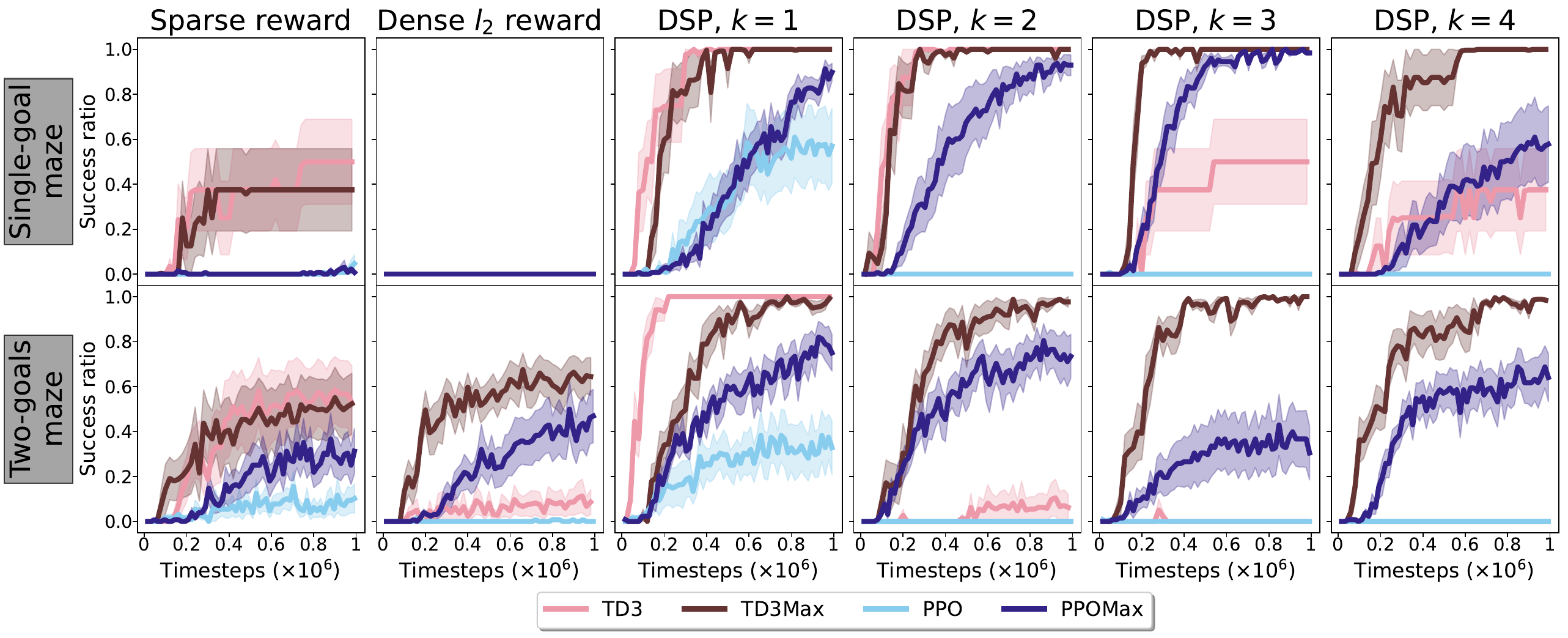}
\caption{Learning curves of TD3, max-reward TD3, PPO, and max-reward PPO on two different mazes. The vertical axis is the success ratio, i.e., whether the goal was reached during the episode. The shaded area is the standard error of the mean. The horizontal axis is the total environmental timesteps in millions. For each maze, we present results for six different reward functions (columns).}
\label{fig:sp-experiment}
\end{center}
\end{figure*}
The policy gradient theorems allow us to use various algorithms from standard RL, {such as REINFORCE} ~\cite{williams1992}, {A2C}~\cite{sutton1999policy}, {A3C}~\cite{mnih2016}, {TRPO}~\cite{schulman2015}, {PPO}~\cite{schulman2017}, {DDPG}~\cite{lillicrap2016}, and {TD3}~\cite{fujimoto2018}, to optimize maximum rewards.
In this work, we focus on PPO and TD3, as they are considered to be the best-performing algorithms within their corresponding families. For max-reward PPO, the only difference compared to the standard version is that the advantage estimation uses max-reward returns. For max-reward TD3, the target value for the $Q$ functions is computed using the max-reward Bellman optimality operator \eqref{def:opt-maxr-bellman}. In Appendix, we provide descriptions of max-reward TD3 and PPO in pseudocode.

\section{Experiments}
To empirically evaluate the benefits of using maximum instead of cumulative reward, we compare the max-reward TD3 and PPO with their cumulative counterparts using two goal-reaching environments from Gymnasium-Robotics \cite{gymnasium_robotics2023github} under different dense reward functions.

\subsection{Maze with shortest path rewards}
First, we consider the Maze environment from Gym Robotics \cite{gymnasium_robotics2023github} illustrated in Figure \ref{fig:mazes}, where the agent controls a ball by applying acceleration in two dimensions. The objective is to reach the goal position in the maze. Episodes last 1000 timesteps and there are no terminal states.
We use two mazes: \textit{single-goal} maze, where the goal is always in the same location, and the \textit{two-goals} maze where at each episode the goal location is chosen randomly from the two possible options.
The main metric in this environment is \textit{success ratio} -- a binary value indicating whether the goal was reached during the episode.

We consider several reward functions that induce the same optimal policy of reaching the goal state:
\begin{itemize}
    \item[1.] \textit{Sparse reward} -- only reaching the goal is rewarded with $r=1.$
    \item[2.] \textit{Dense $l_2$ reward} -- default dense reward, defined as the exponent of the negative of the $l_2$-distance to the goal. Reaching the goal is rewarded with $r=1.$
    \item[3.]\textit{Discrete shortest path (DSP)} -- our custom reward that represents the true, topology aware distance to the goal. To compute it, the maze is split into $n\times m$ cells. Then, the distance matrix $D\in\R^{n\times m }$ is computed such that for each cell $(i,j)$, $D[i, j]$ is the number of cells between $(i,j)$ and the goal-containing cell. The DSP reward with parameter $k\in\mathbb{N}$ is then defined as 
    $$r^k_{dsp}(i,j)=\begin{cases}
        \beta^{D[i,j] + 1},\;\text{ if } D[i,j]=0 \mod k\\
        0,\quad\quad\quad\;\,\text{ otherwise }
    \end{cases}$$
    where $\beta\in (0, 1)$ is a hyperparameter. The value of $k$ controls the sparsity of the reward, i.e., for larger $k$ fewer cells have a non-zero reward.  Reaching the goal is rewarded with $r=1.$
\end{itemize}

\begin{figure*}[t!]
    \begin{center}
    \includegraphics[scale=0.4]{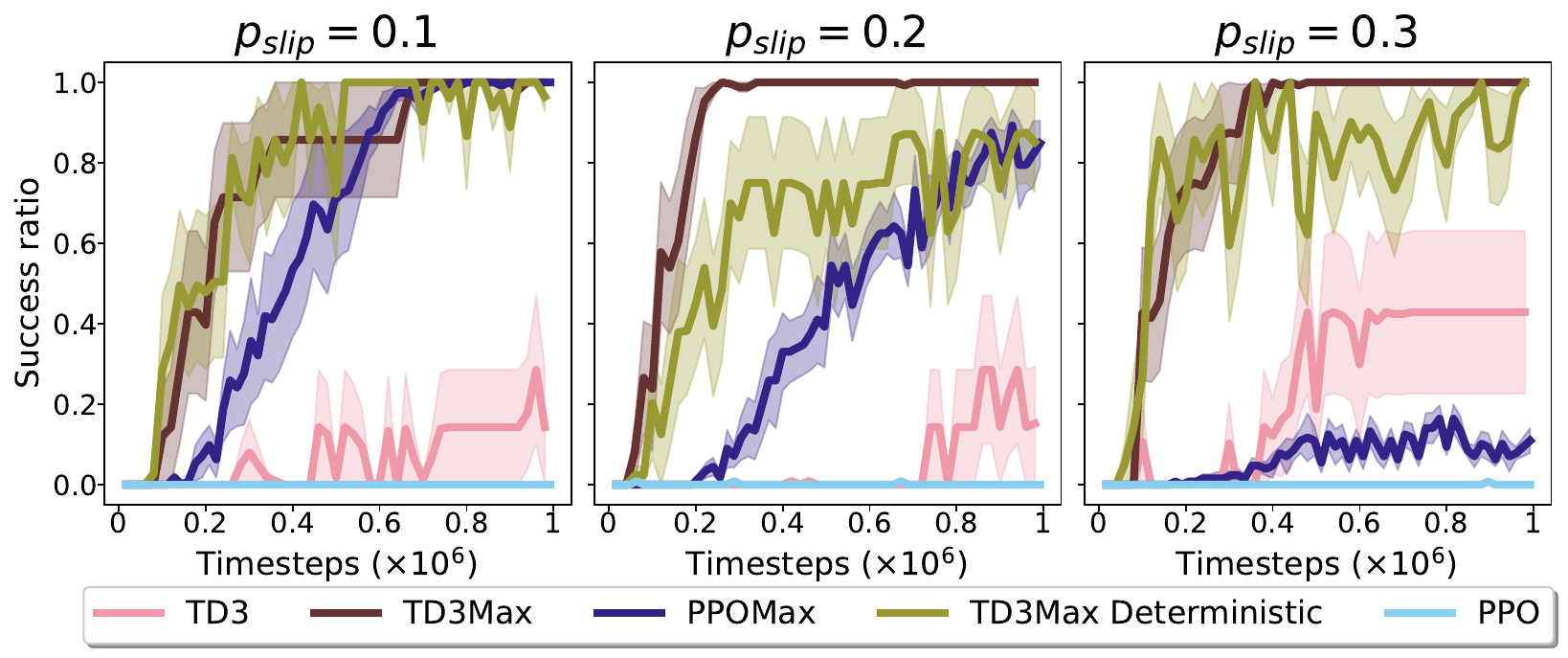}
    \caption{Learning curves of TD3, max-reward TD3, deterministic max-reward TD3, PPO, and max-reward PPO on a stochastic version of the single-goal maze with DSP reward, $k=3$. The vertical axis is the success ratio, the shaded area is the standard error of the mean. The horizontal axis is the total environmental timesteps. The results confirm that our max-reward methods work in stochastic environments.}
    \label{fig:stochastic-maze}
    \end{center}
\end{figure*}
For the DSP reward, we first tune the value of $\beta$ by running standard TD3 and PPO on the single-goal maze. We set $k=1$ and run 10 random seeds for each algorithm for $\beta\in \{0.65, 0.7, 0.75, 0.8, 0.85, 0.9, 0.95\}$. Additionally, we test the negative version of the DSP reward, $r^k_{dsp}(i,j) - 1$, which, in theory, should cause better exploration.
For TD3, the best-performing reward was the negative DSP with $\beta_{TD3}=0.9$, and for PPO -- negative DSP with $\beta_{PPO}=0.95.$  In all other runs involving DSP reward we use these values of $\beta.$ 

Finally, we compare TD3, PPO, max-reward TD3, and max-reward PPO on the single-goal and two-goals mazes using sparse, dense $l_2$, and DSP reward for $k=1,2,3,4$ (for cumulative methods, negative DSP reward is used). Figure \ref{fig:sp-experiment} demonstrates the learning curves. The sparse reward performs inconsistently due to insufficient exploration. Dense $l_2$ reward has local maximums (especially in the single-goal maze) and its performance greatly depends on the maze topology. The DSP reward, which represents the true distance to the goal, overall performs better.

Importantly, we see that the max-reward approaches work for all values of $k,$ while the standard RL methods do not. For larger $k$, the reward becomes sparser, and cumulative approaches tend to converge to suboptimal policies. We believe that the nature of this phenomenon is the same as in the chain environment example discussed in Section~\ref{sec:det-max-r}. Specifically, the Maze environment can be seen as a larger chain with multiple intermediate rewards. During training, all methods quickly learn to stay in one of the cells with non-zero reward. Then, to update the policy, samples of transitions to a better state are needed. For larger $k,$ these transitions become less frequent, as the cells with non-zero reward become further from one another. In line with the chain environment results, max-reward methods require fewer such transitions and therefore perform more efficiently. 

Another potential reason for the superiority of max-reward methods lies in the way how they handle local optima. Since max-reward policy is conditioned on the discounted max-reward so far, $y$, it has no incentive to stay in the local optima. As $y$ ``remembers'' the reward at a local optimum, any trajectory leaving this optimum is at least as good as staying in the optimum. Combined with exploration techniques, e.g., entropy regularization in PPO, this causes the agent to leave local optima after visiting them.

\paragraph{Stochastic Maze.} One of the strengths of our RL formulation is that it works with stochastic environments and/or policies. To experimentally verify that, we conduct an additional experiment using a stochastic variation of the single-goal maze. Specifically, we introduce a parameter $p_{slip}$ which regulates the level of stochasticity. Whenever the agent makes an action, it is replaced with a random action with probability $p_{slip}.$ We compare max-reward and standard versions of TD3 and PPO on this environment. Additionally, we implement and test \textit{deterministic} max-reward TD3 \cite{gottipati2020maximum}. 
In this experiment, we use the DSP reward with $k=3$, as it is a case where the max-reward paradigm demonstrates improvement over standard RL in a deterministic Maze. The results presented in Figure \ref{fig:stochastic-maze} confirm the theory: our max-reward TD3 solves this stochastic environment while the deterministic max-reward TD3 is highly inconsistent. Therefore, we conclude that our method indeed can be used for stochastic environments.

\subsection{Fetch environment}
\begin{figure*}[t!]
    \centering
    \includegraphics[scale=0.45]{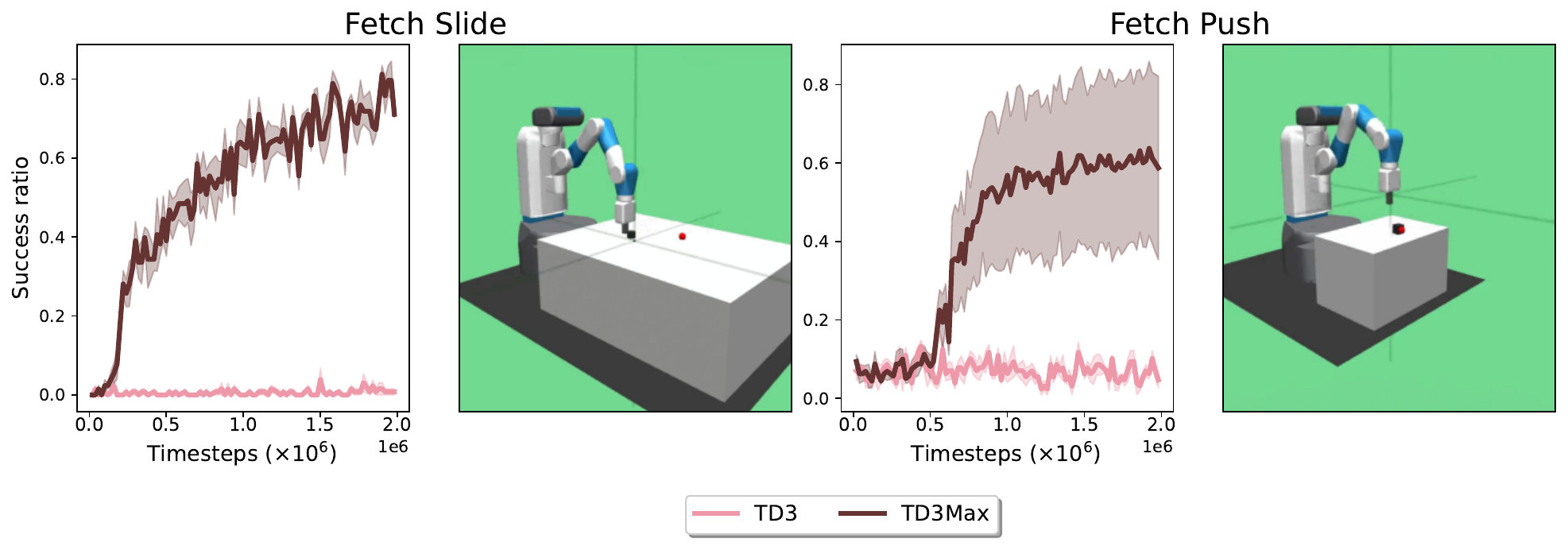}
    \caption{Success ratio for standard (light red) and max-reward (dark-red) TD3 in \textit{Fetch Slide} (left) and \textit{Fetch Push} (right) environments.}
    \label{fig:fetch}
\end{figure*}
In the second experiment, we consider more challenging robotics problems. Specifically, we study the \textit{Fetch-Slide} and \textit{Fetch-Push} environments depicted in Figure \ref{fig:fetch}. The agent controls a 7-DoF manipulator and its goal is to move the puck into the target location. In \textit{Fetch-Slide}, the goal is located beyond agent's reach and hence it needs to slide the puck into the goal. In \textit{Fetch-Push}, the agent needs to push the puck into the goal which can be anywhere on the platform. Each episode is truncated after 100 timesteps and there are no terminal states. We use the standard dense reward for this problem defined as negative of the $l_2-$distance between the puck and the goal.
The performance metric for this environment is again the success ratio -- a binary value that indicates whether the goal was reached during the episode. 
This environment is known to be challenging for standard RL and it cannot be solved without special approaches \cite{plappert2018multi}. 

The plot in Figure \ref{fig:fetch} demonstrates that the max-reward TD3 achieves a goal-reaching policy in both environments, while the standard version, in line with the prior work, fails to learn completely. We believe that this happens due to the difference in bootstrapping mentioned earlier. The environment is complex and multidimensional and the goal-reaching transitions are rare which makes the learning problem really hard for the standard methods. We believe that this experiment shows the great potential of max-reward RL in more realistic goal-reaching environments.

\section{Conclusions and future work}
In this work, we provide a theoretical description of the max-reward reinforcement learning paradigm and verify it experimentally. Our theoretical contributions include a novel formulation of the max-reward value functions and a Bellman-like contraction operator that enables efficient learning. Besides, we prove the policy gradient theorems for max-reward policies and hence enable using the state-of-the-art RL algorithms in the context of max-reward RL.

In the experiments with two robotic environments, we show that max-reward RL works better for sparse reward problems with surrogate dense reward. This result confirm our intuition that maximum reward is a better choice for goal-reaching environments. Moreover, we demonstrate that our max-reward RL, unlike prior work, is also consistent in stochastic environments.

Qualitatively, we believe that the main strengths of the max-reward algorithms can be summarized as follows:
\begin{itemize}
    \item [1.] {In max-reward RL, bootstrapping works differently than in the standard RL. Specifically, it allows for more efficient propagation of reward from the goal states.}
    \item [2.] Max-reward agents are more prone to getting stuck in local optima. Since the maximum reward obtained so far is a part of the extended state space, the agents do not have any incentive to stay in these optima.
    \item [3.] Due to the auxiliary variable $y,$ max-reward value functions are inherently distributional. As reported in prior work \cite{bellemare2017distributional}, learning distributional value functions can have positive impact even in deterministic problems. 
\end{itemize}

In future work, we aim to study how max-reward RL can be combined with the existing methods for automated reward design and explore its potential in other problems. 

\section*{Acknowledgements}
We would like to acknowledge Delft University of Technology for providing the resources and support necessary for this research. The collaborative environment greatly contributed to the development and success of this work. {In particular, we are very thankful to Jinke He for reading the draft of this paper and providing helpful feedback.}

Additionally, we our thank the anonymous reviewers for their thorough and constructive feedback. Their insightful comments and suggestions have significantly improved the quality of this paper.

\section*{Impact Statement}
This paper presents work whose goal is to advance the field of Reinforcement Learning. There are many potential societal consequences of our work, none which we feel must be specifically highlighted here.

\newpage
\newpage

\bibliographystyle{icml2024}
\bibliography{references}

\begin{thebibliography}{38}
\providecommand{\natexlab}[1]{#1}
\providecommand{\url}[1]{\texttt{#1}}
\expandafter\ifx\csname urlstyle\endcsname\relax
  \providecommand{\doi}[1]{doi: #1}\else
  \providecommand{\doi}{doi: \begingroup \urlstyle{rm}\Url}\fi

\bibitem[Bellemare et~al.(2017)Bellemare, Dabney, and Munos]{bellemare2017distributional}
Bellemare, M.~G., Dabney, W., and Munos, R.
\newblock A distributional perspective on reinforcement learning, 2017.

\bibitem[Bellman(1954)]{bellman1954some}
Bellman, R.
\newblock Some applications of the theory of dynamic programming—a review.
\newblock \emph{Journal of the Operations Research Society of America}, 2\penalty0 (3):\penalty0 275--288, 1954.

\bibitem[Berner et~al.(2019)Berner, Brockman, Chan, Cheung, Debiak, Dennison, Farhi, Fischer, Hashme, Hesse, et~al.]{berner2019dota}
Berner, C., Brockman, G., Chan, B., Cheung, V., Debiak, P., Dennison, C., Farhi, D., Fischer, Q., Hashme, S., Hesse, C., et~al.
\newblock Dota 2 with large scale deep reinforcement learning.
\newblock \emph{arXiv preprint arXiv:1912.06680}, 2019.

\bibitem[Burda et~al.(2018)Burda, Edwards, Storkey, and Klimov]{burda2018exploration}
Burda, Y., Edwards, H., Storkey, A., and Klimov, O.
\newblock Exploration by random network distillation, 2018.

\bibitem[Cui \& Yu(2023)Cui and Yu]{cui2023reinforcement}
Cui, W. and Yu, W.
\newblock Reinforcement learning with non-cumulative objective.
\newblock \emph{IEEE Transactions on Machine Learning in Communications and Networking}, 2023.

\bibitem[de~Lazcano et~al.(2023)de~Lazcano, Andreas, Tai, Lee, and Terry]{gymnasium_robotics2023github}
de~Lazcano, R., Andreas, K., Tai, J.~J., Lee, S.~R., and Terry, J.
\newblock Gymnasium robotics, 2023.
\newblock URL \url{http://github.com/Farama-Foundation/Gymnasium-Robotics}.

\bibitem[Eschmann(2021)]{eschmann2021reward}
Eschmann, J.
\newblock Reward function design in reinforcement learning.
\newblock \emph{Reinforcement Learning Algorithms: Analysis and Applications}, pp.\  25--33, 2021.

\bibitem[Fisac et~al.(2014)Fisac, Chen, Tomlin, and Sastry]{fisac2014reachavoid}
Fisac, J.~F., Chen, M., Tomlin, C.~J., and Sastry, S.~S.
\newblock Reach-avoid problems with time-varying dynamics, targets and constraints, 2014.

\bibitem[Fisac et~al.(2019)Fisac, Lugovoy, Rubies-Royo, Ghosh, and Tomlin]{fisac2019bridging}
Fisac, J.~F., Lugovoy, N.~F., Rubies-Royo, V., Ghosh, S., and Tomlin, C.~J.
\newblock Bridging {Hamilton-Jacobi} safety analysis and reinforcement learning.
\newblock In \emph{2019 International Conference on Robotics and Automation ({ICRA})}, pp.\  8550--8556, May 2019.

\bibitem[Florensa et~al.(2018)Florensa, Held, Geng, and Abbeel]{florensa2018automatic}
Florensa, C., Held, D., Geng, X., and Abbeel, P.
\newblock Automatic goal generation for reinforcement learning agents, 2018.

\bibitem[Fujimoto et~al.(2018)Fujimoto, van Hoof, and Meger]{fujimoto2018}
Fujimoto, S., van Hoof, H., and Meger, D.
\newblock Addressing function approximation error in actor-critic methods.
\newblock In \emph{Proceedings of Machine Learning Research (ICML)}, volume~80, pp.\  1587--1596, 2018.

\bibitem[Ghosh et~al.(2020)Ghosh, Gupta, Reddy, Fu, Devin, Eysenbach, and Levine]{ghosh2020learning}
Ghosh, D., Gupta, A., Reddy, A., Fu, J., Devin, C., Eysenbach, B., and Levine, S.
\newblock Learning to reach goals via iterated supervised learning, 2020.

\bibitem[Gottipati et~al.(2020)Gottipati, Pathak, Nuttall, Sahir, Chunduru, Touati, Subramanian, Taylor, and Chandar]{gottipati2020maximum}
Gottipati, S.~K., Pathak, Y., Nuttall, R., Sahir, Chunduru, R., Touati, A., Subramanian, S.~G., Taylor, M.~E., and Chandar, S.
\newblock Maximum reward formulation in reinforcement learning, 2020.

\bibitem[Haarnoja et~al.(2023)Haarnoja, Moran, Lever, Huang, Tirumala, Wulfmeier, Humplik, Tunyasuvunakool, Siegel, Hafner, Bloesch, Hartikainen, Byravan, Hasenclever, Tassa, Sadeghi, Batchelor, Casarini, Saliceti, Game, Sreendra, Patel, Gwira, Huber, Hurley, Nori, Hadsell, and Heess]{haarnoja2023learning}
Haarnoja, T., Moran, B., Lever, G., Huang, S.~H., Tirumala, D., Wulfmeier, M., Humplik, J., Tunyasuvunakool, S., Siegel, N.~Y., Hafner, R., Bloesch, M., Hartikainen, K., Byravan, A., Hasenclever, L., Tassa, Y., Sadeghi, F., Batchelor, N., Casarini, F., Saliceti, S., Game, C., Sreendra, N., Patel, K., Gwira, M., Huber, A., Hurley, N., Nori, F., Hadsell, R., and Heess, N.
\newblock Learning agile soccer skills for a bipedal robot with deep reinforcement learning, 2023.

\bibitem[Hsu et~al.(2021)Hsu, Rubies-Royo, Tomlin, and Fisac]{hsu2021safelive}
Hsu, K.-C., Rubies-Royo, V., Tomlin, C.~J., and Fisac, J.~F.
\newblock Safety and liveness guarantees through {Reach-Avoid} reinforcement learning.
\newblock December 2021.

\bibitem[Ibarz et~al.(2018)Ibarz, Leike, Pohlen, Irving, Legg, and Amodei]{ibarz2018reward}
Ibarz, B., Leike, J., Pohlen, T., Irving, G., Legg, S., and Amodei, D.
\newblock Reward learning from human preferences and demonstrations in atari, 2018.

\bibitem[Jensen(1906)]{jensen1906fonctions}
Jensen, J. L. W.~V.
\newblock Sur les fonctions convexes et les in{\'e}galit{\'e}s entre les valeurs moyennes.
\newblock \emph{Acta mathematica}, 30\penalty0 (1):\penalty0 175--193, 1906.

\bibitem[Krakovna et~al.(2020)Krakovna, Uesato, Mikulik, Rahtz, Everitt, Kumar, Kenton, Leike, and Legg]{deepmindblog}
Krakovna, V., Uesato, J., Mikulik, V., Rahtz, M., Everitt, T., Kumar, R., Kenton, Z., Leike, J., and Legg, S.
\newblock Specification gaming: the flip side of ai ingenuity.
\newblock 2020.

\bibitem[Kr{\"o}ner et~al.(2018)Kr{\"o}ner, Picarelli, and Zidani]{Kroner2018-lo}
Kr{\"o}ner, A., Picarelli, A., and Zidani, H.
\newblock Infinite horizon stochastic optimal control problems with running maximum cost.
\newblock \emph{SIAM J. Control Optim.}, 56\penalty0 (5):\penalty0 3296--3319, January 2018.

\bibitem[Lillicrap et~al.(2016)Lillicrap, Hunt, Pritzel, Heess, Erez, Tassa, Silver, and Wierstra]{lillicrap2016}
Lillicrap, T.~P., Hunt, J.~J., Pritzel, A., Heess, N., Erez, T., Tassa, Y., Silver, D., and Wierstra, D.
\newblock Continuous control with deep reinforcement learning.
\newblock In \emph{International Conference on Learning Representations (ICLR)}, 2016.

\bibitem[Mataric(1994)]{mataric1994reward}
Mataric, M.~J.
\newblock Reward functions for accelerated learning.
\newblock In \emph{Machine learning proceedings 1994}, pp.\  181--189. Elsevier, 1994.

\bibitem[Mnih et~al.(2013)Mnih, Kavukcuoglu, Silver, Graves, Antonoglou, Wierstra, and Riedmiller]{mnih2013playing}
Mnih, V., Kavukcuoglu, K., Silver, D., Graves, A., Antonoglou, I., Wierstra, D., and Riedmiller, M.
\newblock Playing atari with deep reinforcement learning.
\newblock \emph{arXiv preprint arXiv:1312.5602}, 2013.

\bibitem[Mnih et~al.(2016)Mnih, Badia, Mirza, Graves, Lillicrap, Harley, Silver, and Kavukcuoglu]{mnih2016}
Mnih, V., Badia, A.~P., Mirza, M., Graves, A., Lillicrap, T., Harley, T., Silver, D., and Kavukcuoglu, K.
\newblock Asynchronous methods for deep reinforcement learning.
\newblock In \emph{Proceedings of The 33rd International Conference on Machine Learning}, volume~48 of \emph{Proceedings of Machine Learning Research}, pp.\  1928--1937. PMLR, 2016.

\bibitem[Ng et~al.(1999)Ng, Harada, and Russell]{ng1999policy}
Ng, A.~Y., Harada, D., and Russell, S.
\newblock Policy invariance under reward transformations: Theory and application to reward shaping.
\newblock In \emph{Icml}, volume~99, pp.\  278--287. Citeseer, 1999.

\bibitem[Pathak et~al.(2017)Pathak, Agrawal, Efros, and Darrell]{pathak2017curiositydriven}
Pathak, D., Agrawal, P., Efros, A.~A., and Darrell, T.
\newblock Curiosity-driven exploration by self-supervised prediction, 2017.

\bibitem[Plappert et~al.(2018)Plappert, Andrychowicz, Ray, McGrew, Baker, Powell, Schneider, Tobin, Chociej, Welinder, et~al.]{plappert2018multi}
Plappert, M., Andrychowicz, M., Ray, A., McGrew, B., Baker, B., Powell, G., Schneider, J., Tobin, J., Chociej, M., Welinder, P., et~al.
\newblock Multi-goal reinforcement learning: Challenging robotics environments and request for research.
\newblock \emph{arXiv preprint arXiv:1802.09464}, 2018.

\bibitem[Quah \& Quek(2006)Quah and Quek]{quah2006}
Quah, K. and Quek, C.
\newblock Maximum reward reinforcement learning: A non-cumulative reward criterion.
\newblock \emph{Expert Systems with Applications}, 31\penalty0 (2):\penalty0 351--359, 2006.
\newblock ISSN 0957-4174.
\newblock \doi{https://doi.org/10.1016/j.eswa.2005.09.054}.

\bibitem[Schulman et~al.(2015)Schulman, Levine, Abbeel, Jordan, and Moritz]{schulman2015}
Schulman, J., Levine, S., Abbeel, P., Jordan, M., and Moritz, P.
\newblock Trust region policy optimization.
\newblock In \emph{Proceedings of Machine Learning Research (ICML)}, volume~37, pp.\  1889--1897, 2015.

\bibitem[Schulman et~al.(2017)Schulman, Wolski, Dhariwal, Radford, and Klimov]{schulman2017}
Schulman, J., Wolski, F., Dhariwal, P., Radford, A., and Klimov, O.
\newblock Proximal policy optimization algorithms.
\newblock \emph{CoRR}, abs/1707.06347, 2017.

\bibitem[Silver et~al.(2014)Silver, Lever, Heess, Degris, Wierstra, and Riedmiller]{silver2014deterministic}
Silver, D., Lever, G., Heess, N., Degris, T., Wierstra, D., and Riedmiller, M.
\newblock Deterministic policy gradient algorithms.
\newblock In \emph{International conference on machine learning}, pp.\  387--395. Pmlr, 2014.

\bibitem[Sutton et~al.(1999)Sutton, McAllester, Singh, and Mansour]{sutton1999policy}
Sutton, R.~S., McAllester, D., Singh, S., and Mansour, Y.
\newblock Policy gradient methods for reinforcement learning with function approximation.
\newblock \emph{Advances in neural information processing systems}, 12, 1999.

\bibitem[Tang et~al.(2017)Tang, Houthooft, Foote, Stooke, Chen, Duan, Schulman, Turck, and Abbeel]{tang2017exploration}
Tang, H., Houthooft, R., Foote, D., Stooke, A., Chen, X., Duan, Y., Schulman, J., Turck, F.~D., and Abbeel, P.
\newblock Exploration: A study of count-based exploration for deep reinforcement learning, 2017.

\bibitem[Towers et~al.(2023)Towers, Terry, Kwiatkowski, Balis, Cola, Deleu, Goulão, Kallinteris, KG, Krimmel, Perez-Vicente, Pierré, Schulhoff, Tai, Shen, and Younis]{gymnasium2023github}
Towers, M., Terry, J.~K., Kwiatkowski, A., Balis, J.~U., Cola, G.~d., Deleu, T., Goulão, M., Kallinteris, A., KG, A., Krimmel, M., Perez-Vicente, R., Pierré, A., Schulhoff, S., Tai, J.~J., Shen, A. T.~J., and Younis, O.~G.
\newblock Gymnasium, March 2023.
\newblock URL \url{https://zenodo.org/record/8127025}.

\bibitem[Trott et~al.(2019)Trott, Zheng, Xiong, and Socher]{trott2019keeping}
Trott, A., Zheng, S., Xiong, C., and Socher, R.
\newblock Keeping your distance: Solving sparse reward tasks using self-balancing shaped rewards, 2019.

\bibitem[Vinyals et~al.(2019)Vinyals, Babuschkin, Czarnecki, Mathieu, Dudzik, Chung, Choi, Powell, Ewalds, Georgiev, et~al.]{vinyals2019grandmaster}
Vinyals, O., Babuschkin, I., Czarnecki, W.~M., Mathieu, M., Dudzik, A., Chung, J., Choi, D.~H., Powell, R., Ewalds, T., Georgiev, P., et~al.
\newblock Grandmaster level in starcraft ii using multi-agent reinforcement learning.
\newblock \emph{Nature}, 575\penalty0 (7782):\penalty0 350--354, 2019.

\bibitem[Wang et~al.(2020)Wang, Zhong, Du, Salakhutdinov, and Yang]{wang2020planning}
Wang, R., Zhong, P., Du, S.~S., Salakhutdinov, R.~R., and Yang, L.
\newblock Planning with general objective functions: Going beyond total rewards.
\newblock \emph{Advances in Neural Information Processing Systems}, 33:\penalty0 14486--14497, 2020.

\bibitem[Williams(1992)]{williams1992}
Williams, R.~J.
\newblock Simple statistical gradient algorithms for connectionist reinforcement learning.
\newblock \emph{Machine Learning}, 8:\penalty0 229--256, 1992.

\bibitem[Yu et~al.(2022)Yu, Ma, Li, and Chen]{yu2022reach}
Yu, D., Ma, H., Li, S.~E., and Chen, J.
\newblock Reachability constrained reinforcement learning.
\newblock May 2022.

\end{thebibliography}

\newpage
\appendix
\onecolumn
\section{Proofs} \label{proofs}
\begin{proof}[Proof of Lemma \ref{lemma:bellman-eq}]
First, we prove the equation for the state value function $\hat{v}^\pi:$
\begin{equation*}
\begin{aligned}
\hat{v}^\pi(s,y)&=
\E_\pi \big[y\lor \hat{G}_t
\big|\substack{s_t=s}\big] = 
\E_\pi\big[y\lor r_{t+1}\lor\gamma\hat{G}_{t+1}
\big|\substack{s_t=s}\big] \\ &=
\gamma\E_\pi \big[y'\lor\hat{G}_{t+1}
\big|\substack{s_t=s}\big] = 
\gamma\E_{\pi}\big[y'\lor y'\lor\hat{G}_{t+1} \big|\substack{s_t=s}\big] \\ &=
\gamma\E_{\substack{a_t\\ s_{t+1}}}\big[y'\lor\hat{v}^\pi(s_{t+1},y') \big|\substack{s_t=s}\big]
\end{aligned}
\end{equation*}
Then, for the state-action value function $\hat{q}^\pi:$
\begin{equation*}
\begin{aligned}
\hat{q}^\pi(s,a,y)&=
\E_\pi \big[y\lor \hat{G}_t
\big|\substack{s_t=s\\a_t=a}\big] = 
\E_\pi\big[y\lor r_{t+1}\lor\gamma\hat{G}_{t+1}
\big|\substack{s_t=s\\a_t=a}\big]
\\ &=
\gamma\E_\pi \big[y'\lor\hat{G}_{t+1}
\big|\substack{s_t=s\\a_t=a}\big] = 
\gamma\E_{\pi}\big[y'\lor y' \lor\hat{G}_{t+1}\big|\substack{s_t=s\\a_t=a}\big]
\\ &=
\gamma\E_{\substack{s_{t+1}\\a_{t+1}}}\big[y'\lor\hat{q}^\pi(s_{t+1}a_{t+1},y') \big|\substack{s_t=s\\a_t=a}\big]
\end{aligned}
\end{equation*}
\end{proof}

\begin{proof}[Proof of Theorem \ref{theorem:contraction}]
First, we demonstrate a simple property of the $\lor$ operator that we will use later. Let $a,x,y\in\R.$ Then, using equation $x\lor y = 0.5(x+y+|x-y|)$, we obtain the following:
\begin{equation}
\label{eq:max-property}
\begin{aligned}
a\lor x - a\lor y & = 
0.5\big(a+x+ |x-a| -a-y-|y-a|\big) 
\\ &=
0.5(x-y+|x-a|-|y-a|) 
\\ & \leq
0.5(x-y +|x-a-(y-a)|)
\\ & =
0.5(x-y+|x-y|) 
\\ & \leq
|x - y|
\end{aligned}
\end{equation}
Now, we can prove that $\hat{\mathcal{T}}^\pi$ is a contraction. We begin with the state-action case.
Let $q, z: \mathcal{S}\times\mathcal{A}\times\R$ be two-real valued functions. Then, we can expand $\|{\hat{\mathcal{T}}}^\pi q-{\hat{\mathcal{T}}}^\pi z\|_\infty$ as follows:
\begin{equation*}
    \begin{aligned}
\|{\hat{\mathcal{T}}}^\pi q-{\hat{\mathcal{T}}}^\pi z\|_\infty & = 
\gamma\sup_{s\in\mathcal{S}, a\in\mathcal{A}, y\in\R } \Big|\E_{\substack{s_{t+1}\\a_{t+1}}}\Big[y'\lor q(s_{t+1},a_{t+1},y') - y'\lor z(s_{t+1},a_{t+1},y')
\big|\substack{s_t=s\\a_t=a}\Big]\Big|  \\ & \leq 
\gamma\sup_{s\in\mathcal{S}, a\in\mathcal{A}, y\in\R }\E_{\substack{s_{t+1}\\a_{t+1}}}\Big[\Big|y'\lor q(s_{t+1},a_{t+1},y') - y'\lor z(s_{t+1},a_{t+1},y')\Big|
\big|\substack{s_t=s\\a_t=a}\Big]|  
\\ & \overset{(\ast)}{\leq}
\gamma\sup_{s_{t+1}\in\mathcal{S}, a_{t+1}\in\mathcal{A}, y'\in\R }
\Big| y'\lor q(s_{t+1},a_{t+1},y') - y'\lor z(s_{t+1},a_{t+1}, y') \Big| 
\\ & = 
\gamma\sup_{s\in\mathcal{S}, a\in\mathcal{A}, y\in\R }
\Big| y\lor q(s,a,y) - y\lor z(s,a, y) \Big| 
\\ & \leq 
\gamma\sup_{s\in\mathcal{S}, a\in\mathcal{A}, y\in\R }\Big|q(s,a,y)-z(s,a,y)\Big| = \|q-z\|_\infty
    \end{aligned}
\end{equation*}
The first inequality follows from the fact that $\Big|\E_x[x]\Big| \leq\E_{x}\Big[|x|\Big]$ and the last inequality follows from Eq.~\eqref{eq:max-property}. In $(\ast)$, we use the following property of the expectation: $\sup_{y}\big\{\E[x|y]\big\}\leq\sup_{x}\{x\}.$
Now, we demonstrate the contraction property for the state value function:
Let $v, u: \mathcal{S}\times\R$ be two-real valued functions. Then, we expand $\|{\hat{\mathcal{T}}}^\pi v-{\hat{\mathcal{T}}}^\pi u\|_\infty$ as follows:
\begin{equation*}
    \begin{aligned}
\|{\hat{\mathcal{T}}}^\pi v-{\hat{\mathcal{T}}}^\pi u\|_\infty & = 
\gamma\sup_{s\in\mathcal{S}, y\in\R } \Big|\E_{\substack{a_t\\s_{t+1}}}\Big[y'\lor v(s_{t+1},y') - y'\lor u(s_{t+1},y')
\big|\substack{s_t=s}\Big]\Big| 
\\ & \leq 
\gamma\sup_{s\in\mathcal{S}, y\in\R }\E_{\substack{a_{t}\\s_{t+1}}}\Big[\Big|y'\lor v(s_{t+1},y') - y'\lor u(s_{t+1},y')\Big|
\big|\substack{s_t=s\\}\Big]  
\\ & \leq
\gamma\sup_{s_{t+1}\in\mathcal{S}, y'\in\R }
\Big| y'\lor v(s_{t+1},y') - y'\lor u(s_{t+1},y') \Big| 
\\ & = 
\gamma\sup_{s\in\mathcal{S}, y\in\R }
\Big| y\lor v(s,y) - y\lor u(s, y) \Big| 
\\ & \leq 
\gamma\sup_{s\in\mathcal{S}, y\in\R }\Big|v(s,y)-u(s,y)\Big| = \|v-u\|_\infty
    \end{aligned}
\end{equation*}
Therefore, the max-reward Bellman operator is a contraction. Hence, by the Banach fixed-point theorem, it has a unique fixed-point(s). From Lemma \ref{lemma:bellman-eq}, we conclude that this is the max-reward value function(s).
\end{proof}

\begin{proof}[Proof of Lemma \ref{theorem:opt-contraction}]
\begin{equation}
    \begin{aligned}
\|{\hat{\mathcal{T}}}^\ast q-{\hat{\mathcal{T}}}^\ast z\|_\infty & = 
\gamma\sup_{s\in\mathcal{S}, a\in\mathcal{A}, y\in\R } \Big|\E_{\substack{s_{t+1}\\}}\Big[y'\lor \max_{a'}q(s_{t+1},a',y') - y'\lor \max_{a'}z(s_{t+1},a',y')
\big|\substack{s_t=s\\a_t=a}\Big]\Big| 
\\ & \leq
\gamma\sup_{s\in\mathcal{S}, a\in\mathcal{A}, y\in\R } \E_{\substack{s_{t+1}\\}}\Big[\Big|y'\lor \max_{a'}q(s_{t+1},a',y') - y'\lor \max_{a'}z(s_{t+1},a',y')\Big|
\big|\substack{s_t=s\\a_t=a}\Big] 
\\ & \leq
\gamma\sup_{s\in\mathcal{S}, a\in\mathcal{A}, y\in\R} \E_{\substack{s_{t+1}\\}}\Big[\max_{a'}\Big|y'\lor q(s_{t+1},a',y') - y'\lor z(s_{t+1},a',y')\Big|
\big|\substack{s_t=s\\a_t=a}\Big] 
\\ & \leq
\gamma\sup_{s_{t+1}\in\mathcal{S}, a_{t+1}\in\mathcal{A}, y'\in\R }
\Big| y'\lor q(s_{t+1},a_{t+1},y') - y'\lor z(s_{t+1},a_{t+1}, y') \Big| 
\\ & = 
\gamma\sup_{s\in\mathcal{S}, a\in\mathcal{A}, y\in\R }
\Big| y\lor q(s,a,y) - y\lor z(s,a, y) \Big| 
\\ & \leq 
\gamma\sup_{s\in\mathcal{S}, a\in\mathcal{A}, y\in\R }\Big|q(s,a,y)-z(s,a,y)\Big| = \|q-z\|_\infty
    \end{aligned}
\end{equation}

\end{proof}

\begin{proof}[Proof of Theorem \ref{th:stochastic-pg}]
First of all, we notice that the max-reward Bellman equation implies another recursive equation for $\hat{v}^\pi$ and $\hat{q}^\pi:$
 \begin{equation}
 \label{eq:simplified-bellman}
    \hat{q}^\pi(s, a, y) = \gamma\E_{\substack{s_{t+1}}} \big[y'\lor\hat{v}^\pi(s_{t+1}, y')\big|\substack{s_t=s}\big]
    =\gamma\E_{\substack{a_t\\s_{t+1}}} \big[\hat{v}^\pi(s_{t+1}, y')\big|\substack{s_t=s\\a_t=a}\big]
\end{equation}
As discussed in the main paper, we use the version with extra $\lor y'$ to enforce boundary conditions. However, we can still use the equation above as it is a property of the max-reward value function.

Before proving the theorem, we introduce simplified notation to improve readability -- for all functions, we use subscripts to denote the input variables. For example, $\hat{v}_t:=\hat{v}^{\hat{\pi}}(s_t, y_t)$.
The proof follows the one for the standard policy gradient theorem. 
We begin by obtaining a recurrent equation for $\nabla_\theta \hat{v}_0:$
\begin{equation*}
    \begin{aligned}
\nabla_\theta\hat{v}_0 & =
\nabla_\theta\big(\int_{a_0}\hat{\pi}_0\hat{q}_0da_0 \big) =
\underbrace{\int_{a_0}(\nabla_\theta\hat{\pi}_0)\hat{q}_0da_0}_{\phi_0} +  \int_{a_0}\hat{\pi}_0(\nabla_\theta q)_0da_0 
\\  & \overset{\text{Eq.\eqref{eq:simplified-bellman}}}{=}
\phi_0 + \gamma\int_{a_0}\hat{\pi}_0\Big(\nabla_\theta 
\int_{s_1,y_1}\hat{p}(s_1,y_1|s_0,y_0,a_0)\hat{v}_1ds_1dy_1\Big)da_0 
\\ & = 
\phi_0 + \gamma \int_{s_1,y_1}\int_{a_0}\hat{\pi}_0\hat{p}(s_1,y_1|s_0,y_0,a_0)(\nabla_\theta\hat{v}_1) ds_1dy_1da_0
\\ & =
\phi_0 + \gamma\int_{s_1,y_1}
\hat{p}_1^{\hat{\pi}}(s_0,y_0, s_1, y_1)
( \nabla_\theta\hat{v}_1) ds_1dy_1 \,,
    \end{aligned} 
\end{equation*}
where we introduced the shorthand 
$\phi_t = \phi(s_t, y_t) = \int_{a} \nabla_\theta \hat\pi(a | s_t, y_t) \, q(s_t, a, y_t) \, da$.
Expanding this recurrence further, we obtain
\begin{equation*}
    \begin{aligned}
\nabla_\theta\hat{v}_0 & =
\sum_{t=0}^\infty \int_{s_t,y_t} \gamma^t \hat{p}_t^{\hat{\pi}}(s_0,y_0, s_t,y_t)\phi(s_t, y_t)ds_tdy_t 
\\ & = 
\int_{s,y}
\Big(\sum_{t=0}^\infty \gamma^t \hat{p}_t^{\hat{\pi}}(s_0,y_0, s,y)\Big)\phi(s, y)ds dy
\\ & \propto
\int_{s,y} \hat{d}^{\hat{\pi}}(s, y|s_0,y_0)\phi(s,y) dsdy 
\\ & = 
\int_{s,y} \hat{d}^{\hat{\pi}}(s, y|s_0,y_0)
\Big(\int_{a}\big(\nabla_\theta\hat{\pi}(a|s,y)\big)\hat{q}(s,a,y)da \Big)dsdy
    \end{aligned} 
\end{equation*}
Above, $ \hat{d}^{\hat{\pi}}(s,y|s_0,y_0)$ is the discounted stationary distribution of $s,y$ for policy $\pi$ given $s_0,y_0$.
Finally, we substitute this formula for $\nabla_\theta\hat{v}_0$ into the definition of $\hat{J}(\theta)$ and conclude the proof:
\begin{equation}
    \begin{aligned}
\nabla_\theta \hat{J}(\theta) & = \int_{s_0,y_0}\hat{p}_0(s_0,y_0)\int_{s,y} \hat{d}^{\hat{\pi}}(s, y|s_0,y_0)
\Big(\int_{a}\big(\nabla_\theta\hat{\pi}(a|s,y)\big)\hat{q}(s,a,y)da \Big)dsdyds_0dy_0
\\ & =
\int_{s,y} \hat{d}^{\hat{\pi}}(s, y)
\Big(\int_{a}\big(\nabla_\theta\hat{\pi}(a|s,y)\big)\hat{q}(s,a,y)da \Big)dsdy
\\ & =
\E_{\substack{s,y\sim \hat{d}^{\hat{\pi}}\\a\sim\hat{\pi}(\cdot|s,y)}}
\big[\hat{q}^{\hat{\pi}}(s,a,y)\nabla_\theta \ln\hat{\pi}(a|s,y)\big]
    \end{aligned}
\end{equation}
\end{proof}

\section{Experimental details} 
\label{exp details}
For all experiments, we used our implementation of TD3 and PPO that we verified on several MuJoco domains. The implementation of the max-reward algorithms is similar to their cumulative versions except for the following differences:
\begin{itemize}
    \item[1.] The input layer of all neural networks has an extra dimension to work with the extended states $(s,y).$
    \item[2.] The output layer of the value networks uses \textit{Tanh} activation and is rescaled to $u\in[0, \bar{R}]$. Then, it is transformed with $ReLU(u-y) + y$ to enforce $\hat{v}^{\pi}(s,y)\geq y.$
\end{itemize}
Hyperparameters of all runs are reported in Tables \ref{tab:maze}-\ref{tab:fetch}.

\begin{table}[h!]
\centering
\begin{tabular}{p{3cm}|p{2cm}|p{2cm}|p{2cm}|p{2cm}|}
    {Parameter} & PPO & PPOMax & TD3 & TD3Max \\
    \hline
    Parallel environments & 16 & 16 & 16 & 16 \\ 
    Discount factor $\gamma$ & 0.99 & 0.999 & 0.99 & 0.995 \\ 
    Learning rate & 3e-4 & 3e-4 & 3e-4 & 3e-4 \\ 
    Lr. annealing & No & No & No & No \\ 
    Entropy weight & 5e-2 & 5e-2 & & \\ 
    Value loss weight & 0.5 & 0.5 & & \\
    Clip coef. & 0.2 & 0.2 & &  \\
    GAE $\lambda$ & 0.95 & 1 & & \\
    Policy update freq. & & & 2 & 2 \\
    Target soft update $\tau$ & & & 0.005 & 0.005 \\
    Expl. noise type & & & pink & pink \\ 
    Expl. noise std & & & 0.7 & 0.7 \\
    Expl. noise clip & & & 0.5 & 0.5 \\
    Target noise scale & & & 0.2 & 0.2 \\ 
    Initial expl. steps & & & 25000 & 25000 \\ 
    Tr. epochs per rollout & 10 & 10 & & \\
    Rollout length & 1024 & 2048 & & \\
    Minibatch size & 32 & 32 & 256 & 256
\end{tabular}
\caption{Hyperparameters for the experiments with Maze environment.}
\label{tab:maze}
\end{table}

\begin{table}[h!]
\centering
\begin{tabular}{p{3cm}|p{2cm}|p{2cm}|}
    {Parameter} & TD3 & TD3Max \\
    \hline
    Parallel environments & 16 & 16 \\ 
    Discount factor $\gamma$ & 0.99 & 0.995 \\ 
    Learning rate & 3e-4 & 3e-4 \\ 
    Lr. annealing & No & No \\ 
    Policy update freq. & 2 & 2 \\
    Target soft update $\tau$ & 0.005 & 0.005 \\
    Expl. noise type & pink & pink \\ 
    Expl. noise std & 0.1 & 0.1 \\
    Expl. noise clip & 0.5 & 0.5 \\
    Target noise scale & 0.2 & 0.2 \\ 
    Initial expl. steps & 25000 & 25000 \\ 
    Minibatch size & 256 & 256
\end{tabular}
\caption{Hyperparameters for the experiments with Fetch environment.}
\label{tab:fetch}
\end{table}

\section{Algorithms}
\label{algos}
\begin{algorithm}[h!]
\begin{algorithmic}[1]
\State Initialize critic networks $\hat{q}_{\phi_1},\hat{q}_{\phi_2}$ and actor network $\mu_\theta$ 
\State Initialize target networks $\phi_1'\gets \phi_1,\,\phi_2'\gets \phi_2,\,\theta'\gets \theta$
\State Initialize replay buffer $\Set{D}$
\For {$episode=1,2,\ldots$}
    \State Initialize $s_0, y_0\sim \hat{p}_0$
    \For {$t=0,1,\ldots,T-1$}
        \State Sample exploration noise $\epsilon_t$
        \State Execute $a_t=\mu(s_t,y_t) + \epsilon_t$  and get  $s_{t+1}, r_{t+1}$
        \State Update $y_{t+1}=(y_t\lor r_{t+1})/{\gamma}$
        \State Save $(s_t, y_t, a_t, s_{t+1}, r_{t+1}, y_{t+1})$ into $\Set{D}$
        \If {initial exploration is over}
            \State Sample a mini-batch of size $N$ from $\Set{D}$
            \State Sample target actions noise $\eta$
            \State $\Tilde{a}\gets\mu_{\theta'}(s',y') + \eta$
            \State $z\gets y'\lor\gamma\min_{i=1,2}\hat{q}_{\phi'_i}(s',\Tilde{a},y')$
            \State Critic loss $L_c=\frac{1}{N}\sum^2_{i=1}(z-\hat{q}_{\phi_i}(s,a,y))^2$ 
            \State Perform gradient update step on $L_c$
            \If {time to update policy}
                \State $L_a\gets \E\big[\hat{q}_{\phi_1}(s,\mu_\theta(s,y),y)\big]$\
                 \State Perform gradient update step on $L_a$
                \State $\phi_i'\gets \tau\phi_i +(1-\tau)\phi'_i,\;i=1,2 $
                \State $\theta'\gets \tau\theta +(1-\tau)\theta $
            \EndIf
        \EndIf
    \EndFor
\EndFor
\caption{Max-reward TD3} 
\label{alg:td3}
\end{algorithmic} 
\end{algorithm}

\begin{algorithm}[h!]
\caption{Max-reward PPO} 
\label{alg:ppo}
\begin{algorithmic}[1]
\State Initialize actor $\hat{\pi}_\theta$ and critic $\hat{v}_\phi$
\For {$iteration=1,2,\ldots$}
     \State Initialize trajectories buffer $\Set{D}$ 
    \For {$actor=1,2,\ldots,N$}
        \State Initialize $s_0, y_0\sim \hat{p}_0$
	\For {$t=0,1,\ldots,T-1$}
            \State Execute $a_t\sim \pi_\theta(\cdot|s_t,y_t)$  and get  $s_{t+1}, r_{t+1}$
             \State Update $y_{t+1}=(y_t\lor r_{t+1})/{\gamma}$
            \State Save $(s_t, y_t, a_t, s_{t+1}, r_{t+1}, y_{t+1})$ into $\Set{D}$
        \EndFor
        \State $\hat{G}_t^n\gets\gamma^n\hat{v}_\phi(s_{t+n},y_{t+n}),\; n=1,\ldots, T-t$
        \State $\hat{G}_t(\lambda)=(1-\lambda)\sum_{n=1}^{T-t}\lambda^{n-1}\hat{G}^n_t$
        \State Compute advantages $\hat{A}_t=\hat{G}_t(\lambda) - \hat{v}_\phi(s_t,y_t)$
    \EndFor
    \For {k=$1,2\ldots K$}
    \State Critic loss: $L_c\gets\frac{1}{T}\sum_{t}(\hat{G}_t^{T-t}-\hat{v}_\phi(s,y))^2$
    \State Actor loss: $L_a\gets L_{PPO}(\pi_\theta, \{\hat{A}_t\}_{t=1}^T)$
    \State Perform gradient update step on $L_a + L_c$
    \EndFor
\EndFor
\end{algorithmic} 
\end{algorithm}

\end{document}